\def\year{2018}\relax
\documentclass[letterpaper]{article} 
\usepackage{aaai18}  
\usepackage{times}  
\usepackage{helvet}  
\usepackage{courier}  
\usepackage{url}  
\usepackage{enumitem}
\usepackage{graphicx}  
\usepackage{xcolor}
\usepackage{amssymb,amsthm,amsmath}
\usepackage{bbm}

\usepackage{soul}

\newtheorem{definition}{Definition}
\newtheorem{theorem}{Theorem}
\newtheorem{lemma}[theorem]{Lemma}
\newtheorem{proposition}[theorem]{Proposition}
\newtheorem{remark}{Remark}

\theoremstyle{definition}
\newtheorem{example}{Example}

\newcommand{\negation}[1]{\operatorname{neg}\left( #1 \right)}
\newcommand{\expect}[1]{\mathbb{E}\left[ #1 \right]}
\newcommand{\foralls}[0]{\forall^*}

\newcommand\xqed[1]{%
  \leavevmode\unskip\penalty9999 \hbox{}\nobreak\hfill
  \quad\hbox{#1}}
\newcommand\demo{\xqed{$\triangle$}}

\begin{document}
%
\title{Quantified Markov Logic Networks\thanks{\emph{Authors appear in strict  alphabetical order. \smallskip}}}%
\author{V\'ictor Guti\'errez-Basulto \\ Cardiff University, UK \\ gutierrezbasultov@cardiff.ac.uk \\ 
\And Jean Christoph Jung \\ University of Bremen, Germany \\ KU Leuven, Belgium\\ jeanjung@uni-bremen.de\\
\And Ond\v{r}ej Ku\v{z}elka \\ KU Leuven, Belgium \\ ondrej.kuzelka@kuleuven.be
}
\maketitle

\begin{abstract}
Markov Logic Networks (MLNs) are well-suited for expressing statistics such as {\it ``with high probability a smoker knows another smoker''} but not for expressing statements such as {\it ``there is a smoker who knows most other smokers''}, which is necessary for modeling, e.g.\ influencers in social networks.
To overcome this shortcoming, we study \emph{quantified MLNs} which generalize MLNs by introducing {\em statistical universal quantifiers}, allowing to express also the latter type of statistics in a principled way.
Our main technical contribution is to show that the standard reasoning tasks in quantified MLNs, maximum a posteriori and marginal inference, can be reduced to their respective MLN counterparts in polynomial time.
\end{abstract}

\section{Introduction}

Markov Logic Networks~\cite{RichardsonD06} extend  first-order logic (FOL) with means to capture uncertainty. This is intuitively achieved by softening the meaning of FOL formulas by associating a weight to them, such that the higher the weight, the higher the probability of the formula to be satisfied. Indeed, MLNs provide a  compact representation of large Markov Networks with repeated substructures. 
The FOL component of MLNs makes them particularly suitable to represent background knowledge of a wide variety of application domains. As a consequence, MLNs have been successfully used to model knowledge in domains  such as natural language~\cite{RiedelM11,VenugopalCGN14}, computer vision~\cite{TranD08} and social network analysis~\cite{ChenKZML13,FarasatNSB15}.

The kind of statistical regularities (that hold for a given problem) encoded by an MLN,  directly  depends on the type of quantifiers available in the language. Since MLNs are based on FOL, they come equipped with the standard FOL quantifiers
$\exists$ and $\forall$. However, it has been observed that the modeling capabilities of these quantifiers might not be appropriate for certain application scenarios that require a kind of quantification describing, for instance \emph{most}, \emph{few}, or \emph{at least $k$} thresholds, for more  details see~\cite{FarnadiBMGC17,MilchZKHK08} and references therein, and Sec.~\ref{sec:relwork} below. 

\noindent In the present paper we investigate \emph{Quantified Markov Logic Networks (QMNLs)}, the extension of classical 
MLNs with a \emph{statistical quantifier $\foralls$.} Indeed, MLNs lack means to describe certain types of statistics, e.g., the proportion of people, that are maximally connected to others. 
%
This type of modeling capabilities might be useful, for instance, in social network analysis to model influencers. As we shall see, with the use of  the $\foralls$ quantifier QMLNs are able to express this type of statistics.

Formally, this is done as follows: instead of weighted formulas $(\varphi(x_1,\ldots,x_k),w)$ from MLNs, we use \emph{weighted quantified sentences} $$(Q_1x_1,\ldots,Q_kx_k:\varphi(x_1,\ldots,x_k),w),$$ where the $Q_i$ can be arbitrary quantifiers $\exists,\forall,\foralls$, and $\varphi$ is a classical FOL formula, that is, the only quantifiers in $\varphi$ are $\exists,\forall$. The semantics is given in terms of maximization ($\exists$), minimization ($\forall$), and expectation ($\foralls$). While the semantics of the former ones is as expected, it is important to note that the semantics for the newly introduced statistical quantifier $\foralls$ corresponds to uniform sampling of grounding substitutions of the respective variables. As a consequence, standard MLNs essentially correspond to the fragment of QMLNs where all $Q_i$ are $\foralls$ since in this case, the respective statistic represents the probability that the formula $\varphi(x_1,\ldots,x_k)$ is true in a given possible world after grounding the variables $x_1,\ldots,x_k$ using a randomly sampled substitution. 

As a concrete example in QMLNs using the universal statistical quantifiers $\foralls$,  we can now measure the proportion of the population that are smokers and are known by one particular smoker, who knows most other smokers: 
\begin{equation*}
(\exists x \foralls y \colon (\mathsf{smoker}(x) \land \mathsf{knows}(x,y) \land \mathsf{smoker}(y)), 10).
\end{equation*}
%
 
\noindent Besides the mentioned practical reasons, the study of QMLNs also has a strong theoretical motivation. Let us recall the relation of MLNs with quantifier-free formulas and max entropy models constrained by statistics based on the random substitution semantics~\cite{bacchus_halpern_koller,schulte}. 
MLNs correspond to the solution of the maximum entropy relational marginal problem, where the modeled statistics are of the form $(\varphi(x_1,\ldots,x_k),p)$ with $\varphi(x_1,\ldots,x_k)$  a quantifier-free FOL formula and $p$ is the probability that a random tuple $(a_1,\ldots,a_k)$ satisfies $\varphi(a_1,\ldots,a_k)$ in the model that we learn from~\cite{kuzelka.aaai.2018}. Given that QMLNs 
have higher expressive power, it is interesting to investigate whether a similar correspondence also exists for QMLNs -- of course given more expressive statistics.  
It is easy to see that previously used techniques lift to restricted QMLNs with quantifier prefix $\foralls$. Here we show that the correspondence also holds for arbitrary QMLNs and the respective statistics.

\medskip \noindent {\bf Objective and Contributions} The main objective of this paper is to introduce  QMLNs as an extension of MLNs with means to express more complex statistics, and to develop technical   foundations for these QMLNs. Our main technical contributions are $(i)$~the establishment of basic properties of QMNLs, analogous to those existing for standard MLNs; $(ii)$ a generalization of  the  random substitution semantics to QMLNs and $(iii)$ a polynomial time translation from QMLNs to MLNs, yielding a polytime reduction of the maximum a posteriori and marginal inference in QMLNs  to their respective variant problems in standard MLNs. $(iv)$ Furthermore, we pinpoint certain  implications of extending MLNs to QMLNs in the context of \emph{symmetric  weighted  first-order  model counting (WFOMC)}.

\smallskip \noindent {\bf Outline of the Paper} After providing some preliminaries in Section~\ref{sec:background}, we introduce the syntax and semantics of quantified Markov logic networks in Section~\ref{sec:QMNLs}. Then, in Section~\ref{sec:initialObs} basic results on the treatment of negative weights and weights tending to infinity are provided. We also present results on the duality of relational marginal problems and QMLNs. Sections~\ref{sec:tranMAP} to~\ref{sec:tranMARG} contain our translations from QMLNs to MLNs, establishing the polytime reductions described above. Section~\ref{sec:TwoVar} discusses the relation between QMLNs restricted to two variables and WFOMC. Section~\ref{sec:relwork} presents related work and Section~\ref{sec:con} conclusions and future work.

\section{Background and Notation}\label{sec:background}

We next provide some basics on First-Order Logic,  Markov Logic Networks and Relational Marginal Problems. 

\subsection{\bf First-Order Logic (FOL)} We give a short review  of the  function-free  fragment of first-order logic (FOL), considered in this paper.
Let $\mathbf C = \{a,b, \ldots \}$ be  a finite  set of \emph{constants}  and $\mathbf V = \{ x, y, \ldots \}$ an infinite set of  \emph{variables}. A \emph{term} $t$ is an element in $\mathbf{C} \cup {\mathbf V}$. An \emph{atom} is an expression of the form $R(t_1, \ldots, t_n)$, where $R$ is a \emph{predicate name} with \emph{arity} $n$ and terms $t_i$.  As usual, a  FOL-\emph{formula} $\varphi$  is constructed from atoms using \emph{logical connectives $\neg,\land, \lor, \Rightarrow$} and \emph{quantifiers $\exists$ and $\forall$}. We assume  the reader is familiar with the standard notions of quantified and free variables, and sentence. Given a formula $\varphi$, a variable $x$ and a constant $a$, we use $\varphi [x/a]$ to denote the result of substituting in $\varphi$ every occurrence of $x$ with $a$. Let  $\mathbf{x}=(x_1, \ldots, x_n)$ and $\mathbf{a} = (a_1, \ldots, a_n)$ be tuples of variables and constants, respectively, we write $\varphi [\mathbf x/ \mathbf a]$ to denote the application of  $\varphi [x_i/a_i]$ for all $1 \leq i \leq n$.
%
The \emph{grounding of a formula $\varphi(\mathbf{x})$ over a domain $\Delta$}, denoted with $\mathsf{gr}(\varphi,\Delta)$, is the set of all possible sentences obtained from $\varphi$ by substituting all its free variables $\mathbf{x}=(x_1,\ldots,x_k)$ with any possible combination of constants from $\Delta$, that is,
$$\mathsf{gr}(\varphi,\Delta)= \{\varphi[\mathbf{x}/\mathbf{a}]\mid \mathbf{a}\in\Delta^k\}.$$

A \emph{vocabulary} $\sigma$ is a finite set of predicate names such that each predicate name $R\in\sigma$ is associated with an \emph{arity} $\mathsf{ar}(R)$.  Given a vocabulary $\sigma$ and a domain $\Delta$, a \emph{$\sigma$-structure over $\Delta$} is any set $\omega$ consisting only of facts of the form $R(a_1,\ldots,a_n)$ such that $R\in\sigma$, $\mathsf{ar}(R)=n$, and $a_1,\ldots,a_n\in \Delta$. We denote with $\Omega(\sigma,\Delta)$ the set of all $\sigma$-structures over $\Delta$ and refer to the members of $\Omega(\sigma,\Delta)$ with \emph{possible worlds}.
Throughout the paper we often do not make the vocabulary explicit; it is then assumed to be just the set of predicate names used in the formulas. 

The semantics of FOL is defined as usual in the context of Markov Logic Networks. More precisely, we use only finite domains $\Delta$ and assume that always $\mathbf{C}\subseteq \Delta$. Note that we refrain from using a many-sorted domain for the sake of simplicity; all our results lift to that case. Formally, we write $\omega\models\varphi$ when a sentence $\varphi$ is satisfied in a structure $\omega$. Given a set $\Phi$ of sentences, we write $\omega \models \Phi$ if $\omega \models \varphi$ for all $\varphi\in\Phi$.

\subsection{Markov Logic Networks (MLNs)}

A \emph{Markov Logic Network (MLN)} is a finite set of weighted formulas $(\varphi,w)$, where $w \in \mathbbm R\cup\{+\infty\}$ is a weight and $\varphi$ is a FOL-formula, possibly with free variables. If $w = +\infty$ then $(\varphi,w)$ is called a \emph{hard constraint}, otherwise a \emph{soft constraint}. As an example, the soft constraint
$$(\mathsf{smoker}(x) \wedge \mathsf{friends}(x,y) \Rightarrow \mathsf{smoker}(y), 10)$$
intuitively asserts that having friends who are smokers makes one more likely to be a smoker as well, given the weight $10$ is positive. 

The semantics of MLNs is defined following~\cite{BroeckMD14}.\footnote{This slightly differs from the  semantics introduced in~\cite{RichardsonD06}, we discuss it in the related work section.} A given MLN $\Phi = \{(\varphi_1,w_1),(\varphi_2,w_2),\dots,(\varphi_k,w_k) \}$ and a domain $\Delta$ describe a probability distribution $p_{\Phi,\Delta}$ over $\Omega(\sigma,\Delta)$. To define the distribution, let $\Phi_S\subseteq \Phi$ denote
the set of soft constraints in $\Phi$ and $\Phi_H$ denote the set of FOL sentences $\varphi'$ obtained from the the hard constraints $(\varphi,+\infty) \in \Phi$ by adding a prefix of universal quantifiers for all free variables in $\varphi$. For instance, if $\Phi = \{ (\mathsf{smoker}(x),+\infty) \}$ then $\Phi_H = \{ \forall x : \mathsf{smoker}(x) \}$. Now, the distribution is defined by taking, for every $\omega\in\Omega(\sigma,\Delta)$, 
%
$$
p_{\Phi,\Delta}(\omega) = \begin{cases} 
\frac{1}{Z} \exp{\left( \sum_{(\varphi,w) \in \Phi_S} w \cdot  N(\varphi, \omega) \right)} & \omega \models \Phi_H \\
0 & \mbox{otherw.}\end{cases}
$$
where $N(\varphi, \omega)$ is the number of sentences $\varphi' \in \mathsf{gr}(\varphi,\Delta)$ such that $\omega \models \varphi'$, and $Z$ is a normalization constant. 

\paragraph{Reasoning Problems} We study the following problems. 

\emph{Maximum a posteriori (MAP):}
\begin{itemize}
    \item {\bf Input:} an MLN $\Phi$ and a domain $\Delta$
    \item {\bf Problem:} determine the world $\omega\in\Omega(\sigma,\Delta)$ maximizing $p_{\Phi,\Delta}(\omega)$.
\end{itemize}

\smallskip \emph{Marginal inference (MARG):}
\begin{itemize}
    \item {\bf Input:} an MLN $\Phi$, a domain $\Delta$ and a FOL sentence $\varphi$
    \item {\bf Problem:} compute the probability $Pr_{\Phi,\Delta}(\varphi)$ of $\varphi$, i.e.\
\begin{align*}
    Pr_{\Phi,\Delta}(\varphi)=P_{\omega\sim p_{\Phi,\Delta}(\omega)}[\omega\models \varphi],
    %
\end{align*}
where a subscript $d\sim D$ refers to sampling $d$ according to a distribution $D$.

\end{itemize}
%

\subsection{Relational Marginal Problems}\label{sec:RMPs}

MLNs containing only \emph{quantifier-free} FOL formulas can be seen as solutions to a maximum entropy problem~\cite{kuzelka.aaai.2018} constrained by {\em statistics} which are based on the \emph{random-substitution semantics}~\cite{bacchus_halpern_koller,schulte}. Such statistics are defined as follows. For a possible world $\omega \in \Omega(\sigma,\Delta)$ and a quantifier-free FOL formula $\varphi(\mathbf{x})$ with $k$ free variables $\mathbf{x}$, the \emph{statistic $Q_\omega(\varphi(\mathbf{x}))$ of $\varphi(\mathbf{x})$} is defined as
$$Q_\omega(\varphi(\mathbf{x})) = \mathbb{E}_{\mathbf{a} \sim \operatorname{Unif}(\Delta^k)} \left[ \mathbbm{1}\left( \omega \models \varphi[\mathbf{x}/\mathbf{a}] \right) \right]$$
where $\operatorname{Unif}(\Delta^k)$ denotes the uniform distribution over elements of the set $\Delta^k$.\footnote{Here, we depart slightly from \cite{kuzelka.aaai.2018} in that we do not require the random substitutions to be injective. 
However, the duality of MLNs and relational marginal problems holds as well in the case of non-injective substitutions.}

Intuitively, the statistics $Q_\omega(\varphi(\mathbf{x}))$ measures how likely it is that the formula $\varphi(\mathbf{x})$ is satisfied in $\omega$ when a random substitution of $\mathbf{x}$ by domain elements is picked. The statistics $Q_\omega(\varphi(\mathbf{x}))$ can then be straightforwardly extended to statistics of probability distributions $Q[\varphi(\mathbf{x})]$. For a distribution $p(\omega)$ over $\Omega(\sigma,\Delta)$, $Q[\varphi(\mathbf{x})]$ is defined as
$$Q[\varphi(\mathbf{x})] = \mathbb{E}_\omega\left[ Q_{\omega}(\varphi(\mathbf{x})) \right] = \sum_{\omega \in \Omega(\sigma,\Delta)} p(\omega) \cdot Q_\omega(\varphi(\mathbf{x})).$$
Now, the \emph{maximum entropy relational marginal problem} is defined as follows:
\begin{itemize}
    \item \textbf{Input}: set of statistics $\{(\varphi_1(\mathbf{x}_1),q_1),\ldots,(\varphi_n(\mathbf{x}_n),q_n)\}$, domain $\Delta$
    
    \item \textbf{Problem}: find a distribution $p^*(\omega)$ over $\Omega(\sigma,\Delta)$ which has maximum entropy and satisfies the constraints 
    
    \begin{center}$Q[\varphi_1(\mathbf{x_1})] = q_1$, $\dots$, $Q[\varphi_n(\mathbf{x}_n)] = q_n$.\end{center}
    
\end{itemize}
To motivate the problem, we note that it has been shown in~\cite{kuzelka.aaai.2018} that
\begin{enumerate}[label=(\roman*)]

    \item the solution of the relational marginal problem is a Markov logic network of the form $\{(\varphi_1(\mathbf{x}_1),w_1),\dots,(\varphi_n(\mathbf{x}_n),w_n)\}$ where the weights $w_i$ are obtained from the dual problem of the maximum entropy relational marginal problem, and
    
    \item if the input probabilities $q_1,\ldots,q_n$ are estimated from data with domain $\Delta'$ with $|\Delta'|=|\Delta|$, then the result of the maximum entropy relational marginal problem coincides with maximum likelihood estimation, the most common approach to weight learning in MLNs~\cite{RichardsonD06}.
    
\end{enumerate}
In a sense, Point~(i) can be viewed as the relational generalization of what has been done in the propositional setting~\cite{optimization_entropy_counting}.
Let us finally remark that Point~(ii) does not hold when $|\Delta| \neq |\Delta'|$ which follows from the results by~\citeauthor{shalizi2013consistency} (\citeyear{shalizi2013consistency}). Hence, the relational marginal view is more general from the statistical point of view; we refer to \cite{kuzelka.aaai.2018} for  details.


\section{Quantified Markov Logic Networks}\label{sec:QMNLs}

We introduce the notion of \emph{Quantified Markov Logic Networks (QMLNs)}, a generalization of standard MLNs  capable of expressing expectations using ``statistical'' quantifiers. In QMLNs, the main ingredients of MLNs -- weighted formulas $(\varphi,w)$ with $w$ a weight and $\varphi$ a FOL formula -- are replaced with weighted quantified sentences $(\alpha,w)$.


\begin{definition}[Quantified Sentence] A \emph{quantified sentence} is a formula $\alpha$ with $$\alpha = Q_1 x_1 \ldots Q_n x_n \colon \varphi(x_1,\ldots,x_n),$$ where each $Q_i$ is a quantifier from $\{\forall,\forall^*,\exists\}$ and $\varphi(x_1,\ldots,x_n)$ is a classical first-order formula with free variables precisely $x_1,\ldots,x_n$.
\end{definition}
Note that every FOL sentence is also a quantified sentence, but conversely a quantified sentence using the quantifier $\forall^*$ is not a FOL sentence. 

\begin{definition}[QMLNs]
  A \emph{Quantified Markov Logic Network (QMLN)} is a finite set $\Phi$ of pairs $(\alpha,w)$ such that $\alpha$ is a quantified sentence and $w\in \mathbbm{R}\cup\{+\infty,-\infty\}$.
\end{definition}

\medskip \noindent  Before we can give the semantics of QMLNs, we give the semantics for quantified sentences. Intuitively, given a quantified sentence $\alpha$ and a possible world $\omega$, we measure the extent to which $\alpha$ is satisfied in $\omega$.



\begin{definition}[Sentence Statistics]\label{def:sentence_statistics}
Let $\omega \in \Omega(\sigma,\Delta)$ be a possible world and $\alpha$ be a quantified sentence. Then the $\alpha$-statistic of $\omega$, denoted $Q_\omega(\alpha)$, is defined as follows:
\begin{itemize}
    
    \item if $\alpha$ is an FOL sentence, then 
    \begin{equation}\label{eq:base-case}
    Q_{\omega}(\alpha) = 
    \mathbbm{1}\left( \omega \models \alpha \right),
    \end{equation}
    
    \item if $\alpha = \forall x \colon \alpha'(x)$ is not an FOL sentence, then 
    
    \begin{equation}
    Q_\omega(\alpha) = \min_{a \in \Delta} Q_\omega(\alpha'[x/a]),
    \end{equation}
    
    \item if $\alpha = \exists x \colon \alpha'(x)$ is not an FOL sentence, then 
    \begin{equation}
    Q_\omega(\alpha) = \max_{a \in \Delta} Q_\omega(\alpha'[x/a]),
    \end{equation}
    
    \item if $\alpha = \foralls x \colon  \alpha'(x)$ is not an FOL sentence, then 
    \begin{equation}\label{eq:avg}
    Q_\omega(\alpha) = \frac{1}{|\Delta|}\sum_{a \in \Delta} Q_\omega(\alpha'[x/a]).
    \end{equation}
\end{itemize}
\end{definition}
\noindent Note that the case (\ref{eq:base-case}) in the above definition is only applied when $\alpha$ is a classical FOL formula, that is, when it does not contain any $\forall^*$ quantifiers and, as such, it serves as a base case of the recursive definition. Moreover, the case of the quantifier $\foralls$ in~\eqref{eq:avg} above can be alternatively expressed as
\begin{equation}\label{eq:avg_as_exp}
    Q_\omega(\alpha) = \mathbb{E}_{a \sim \operatorname{Unif}(\Delta)}\left[{Q_\omega(\alpha'[x/a])}\right]
\end{equation}
where the expectation is w.r.t.\ a uniform distribution of $a$ over $\Delta$. From this we see that the definition of statistics $Q_\omega(\alpha)$ given by Definition \ref{def:sentence_statistics} generalizes that of statistics based on random substitution semantics, cf.\ Section \ref{sec:RMPs}.

\begin{remark}\label{rem:1} 
We can easily check the following property of sentence statistics. Let $\alpha$ be a sentence and $\omega \in \Omega(\sigma,\Delta)$ be a possible world. If $\alpha'$ is obtained from $\alpha$ by replacing every quantifier $\foralls$ by its classical counterpart $\forall$ then
$$Q_\omega(\alpha) = 1 \mbox{ iff } \omega \models \alpha'.$$
As a result of this, we will sometimes abuse notation and write $\omega \models \alpha$ when $Q_{\omega}(\alpha) = 1$ even if $\alpha$ is not an FOL sentence.
\end{remark}
\begin{example}\label{example:first}
In  classical first-order logic, the sentence
$$\exists x \forall y \colon  \mathsf{knows}(x,y)$$
asserts that there is someone who knows everyone else (e.g. in a social network). If we replace $\forall y$ by $\foralls y$, we get a quantified sentence
\begin{equation}\label{eq:knows}
\alpha= \exists x \foralls y \colon  \mathsf{knows}(x,y)
\end{equation}%
which relaxes the hard constraint. Indeed, its associated statistic $Q_{\omega}(\exists x \foralls y \colon \mathsf{knows}(x,y))$ measures the maximal proportion of people known by a single domain element. In graph-theoretical terms, this corresponds to the maximum out-degree of domain elements.
Note that we could not directly express the same statistics in normal MLNs since, using normal MLNs, we could only express statistics corresponding to the sentence $\foralls x \exists y : \mathsf{knows}(x,y)$, which intuitively measures the proportion of people who know at least one person. As we show later in the paper, it is possible to express MLNs with constraints encoding the same statistics but in order to do that we will have to enlarge the vocabulary $\sigma$, introducing additional predicates. 
\demo
\end{example}




We now have almost all the ingredients to define the semantics of QMLNs. What remains is to extend the definition of $\Phi_H$, i.e. the hard constraints. Given a QMLN $\Phi = \{(\alpha_1,w_1),(\alpha_2,w_2),\dots,(\alpha_k,w_k) \}$, we define $\Phi_H$ for QMLNs as follows. First, we define $\Phi_H^{+\infty}$ to be the set of FOL sentences obtained from weighted sentences $(\alpha,+\infty) \in \Phi$ by replacing all $\forall^*$ quantifiers by the classical $\forall$ quantifiers. Second, we define $\Phi^{-\infty}_H$ to be the set of FOL sentences $\alpha'$ that are obtained from weighted sentences $(\alpha,-\infty) \in \Phi$ where $\alpha = Q_1 x_1,\dots,Q_n x_n : \varphi(x_1,\dots,x_n)$, as $\alpha' = \widetilde{Q_1} x_1, \dots, \widetilde{Q_n} x_n : \neg \varphi(x_1,\dots,x_n)$ where $\widetilde{\forall^*} = \forall$, $\widetilde{\forall} = \exists$ and $\widetilde{\exists} = \forall$. Finally, we define $\Phi_H = \Phi^{-\infty}_H \cup \Phi^{+\infty}_H$.

Next we define the semantics of QMLNs.

\begin{definition}[Semantics of QMLNs]
Given a QMLN $\Phi = \{(\alpha_1,w_1),(\alpha_2,w_2),\dots,(\alpha_k,w_k) \}$ and a domain $\Delta$, the probability of a possible world $\omega \in \Omega(\sigma,\Delta)$ is defined as:
$$
p_{\Phi,\Delta}(\omega) = \begin{cases} 
\frac{1}{Z} \exp{\left( \sum_{(\alpha,w) \in \Phi_S} w \cdot Q_{\omega}(\alpha) \right)} & \omega \models \Phi_H \\
0 & \mbox{otherwise}\end{cases}
$$
where $Q_{\omega}(\alpha)$ is the $\alpha$-statistic of $\omega$, and
$Z$ is a normalization constant.
\end{definition}

We illustrate the semantics by continuing Example~1. 

\smallskip\noindent\textbf{Example~1 (continued).} Consider again the quantified sentence $\alpha$ from Equation~\eqref{eq:knows} in Example~1. If you include the weighted quantified sentence $(\alpha,w)$ for some $w>0$ in a QMLN, worlds in which there is an individual who knows most of the people get a higher probability than worlds for which this is not the case. As a result, worlds that have an ``influencer'' (and are thus closer to a social network) are considered more likely. \demo

\begin{definition}[Marginal query problem]\label{def:marginal_query}
Let $\alpha$ be a sentence and $p_{\Phi,\Delta}(\omega)$ be the probability distribution over $\Omega(\sigma,\Delta)$ induced by the QMLN $\Phi$ and domain $\Delta$. The \emph{marginal query problem} is to compute the marginal probability defined as:
$$Q_{\Phi,\Delta}[\alpha] = \mathbb{E}_\omega \left[ Q_\omega(\alpha) \right] = \sum_{\omega \in \Omega(\sigma,\Delta)} p_{\Phi,\Delta}(\omega) \cdot Q_\omega(\alpha).$$
\end{definition}

\begin{remark}\label{remark:qfol}
If $\alpha$ is a sentence that does not contain $\foralls$ quantifiers and $p_{\Phi,\Delta}(\omega)$ is the induced distribution over $\Omega(\sigma,\Delta)$ then
$Q_{\Phi,\Delta}[\alpha] = P_{\omega \sim p_{\Phi,\Delta}(\omega)}\left[ \omega \models \alpha \right]$, that is, it coincides with the marginal probability of $\alpha$. Thus,   Definition~\ref{def:marginal_query} generalizes the classical definition of marginal inference.
\end{remark}


\section{Initial Observations}\label{sec:initialObs}

The goal of this section is to make some initial observations about QMLNs, which will be either of independent interest or exploited in some later proof. 

First, it is easy to see that QMLNs generalize MLNs in the sense that we can view a weighted formula $(\varphi,w)$ as a quantified sentence with implicit $\foralls$-quantification over all free variables of $\varphi$. More formally, we have:
\begin{proposition}\label{prop:general}
  Let $\Phi_0$ be an MLN and obtain a QMLN $\Phi$ from $\Phi_0$ by replacing every weighted formula $(\varphi(x_1,\ldots,x_n),w)\in\Phi_0$ with the weighted quantified sentence $(\foralls x_1\ldots\foralls x_n \colon \varphi(x_1,\ldots,x_n), w \cdot |\Delta|^n)$. Then, for every $\Delta$ and $\omega\in\Omega(\sigma,\Delta)$, we have $p_{\Phi_0,\Delta}(\omega)=p_{\Phi,\Delta}(\omega)$.
\end{proposition}

\subsection{Negation in QMLNs}

It is well-known that in classical MLNs it is without loss of generality to assume positive weights. We show an analogous property of QMLNs.

\begin{definition}[Negation]\label{def:negation}
We define the negation $\operatorname{neg}(\alpha)$ of  quantified sentences $\alpha=Q_1x_1\ldots Q_nx_n \colon \varphi(x_1,\ldots,x_n)$ by taking
$$\operatorname{neg}(\alpha)=\overline Q_1x_1\ldots \overline Q_nx_n \colon \neg\varphi(x_1,\ldots,x_n),$$
where $\overline \exists$ is $\forall$, $\overline\forall$ is $\exists$, and $\overline{\foralls}$ is $\foralls$.
%
  
    
    

\end{definition}

\noindent It is easy to check that $\operatorname{neg}(\operatorname{neg}(\alpha)) = \alpha$.
Next we illustrate the way negation works in our setting on a concrete example.

\begin{example}\label{exa:neg}
Let us see what happens if we take the sentence $\exists x \foralls y\colon \mathsf{knows}(x,y)$ from Example \ref{example:first} and negate it. Using Definition \ref{def:negation}, we obtain
$$\operatorname{neg}\left( \exists x \foralls y\colon \mathsf{knows}(x,y) \right) = \forall x \foralls y\colon \neg \mathsf{knows}(x,y).$$
For the statistic $Q_{\omega}\left(\forall x \foralls y\colon \neg \mathsf{knows}(x,y)\right)$ we have
\begin{align*}
&    Q_{\omega}\left(\forall x \foralls y\colon \neg \mathsf{knows}(x,y)\right) \\
&\quad\quad    = \min_{t \in \Delta} \frac{1}{|\Delta|} \sum_{u \in \Delta} \mathbbm{1}\left( \omega \models \neg \mathsf{knows}(t,u) \right) \\
&\quad\quad   = \min_{t \in \Delta} \frac{1}{|\Delta|} \sum_{u \in \Delta} \left(1- \mathbbm{1}\left( \omega \models \mathsf{knows}(t,u) \right)\right) \\
&\quad\quad   = 1 - \max_{t \in \Delta} \frac{1}{|\Delta|} \sum_{u \in \Delta} \mathbbm{1}\left( \omega \models \mathsf{knows}(t,u) \right) \\
&\quad\quad   = 1 - Q_{\omega}(\exists x \foralls y\colon \mathsf{knows}(x,y)).
\end{align*}
\demo
\end{example}
\noindent In Example~\ref{exa:neg}, the statistic of a negation of a sentence $\alpha$ turns out to be equal to one minus the statistic of that sentence, which is intuitively desirable. By repeatedly applying the shown argument, one  can show that this holds in general:

\begin{proposition}\label{prop:negation}
For any sentence $\alpha$ and any possible world $\omega \in \Omega(\sigma,\Delta)$ the following holds: 
$$Q_{\omega}(\negation{\alpha}) = 1-Q_\omega(\alpha).$$
\end{proposition}

Next we show that the same distribution represented by a QMLN $\Phi$ can be represented by another QMLN in which we replace some of the sentences by their negations while also inverting the signs of their respective weights. To show that we will need the next  lemma.

\begin{lemma}\label{lemma:constant_difference}
Let $\Delta$ be a finite domain and $\Phi$ a QMLN. Let $\alpha$, $\beta$ be two sentences and $\Phi_\alpha = \{(\alpha,w) \} \cup \Phi$, $\Phi_\beta = \{ (\beta,w) \} \cup \Phi$ and $\Phi_\beta' = \{ (\beta,-w) \} \cup \Phi$ where $w$ is finite. Then
\begin{enumerate}
    \item if
$Q_{\omega}(\alpha) = Q_{\omega}(\beta) + C_\Delta$
for all $\omega \in \Omega(\sigma,\Delta)$ 
for some constant $C_\Delta$, then 
$p_{\Phi_{\alpha},\Delta}(\omega) = p_{\Phi_\beta,\Delta}(\omega).$
\item if
$Q_{\omega}(\alpha) = -Q_{\omega}(\beta) + C_\Delta$
for all $\omega \in \Omega(\sigma,\Delta)$ for some constant $C_\Delta$, then 
$p_{\Phi_{\alpha},\Delta}(\omega) = p_{\Phi_\beta',\Delta}(\omega).$
\end{enumerate}
\end{lemma}
\begin{proof}
For the first case, we have
\begin{align*}
    &\frac{p_{\Phi_{\alpha},\Delta}(\omega)}{p_{\Phi_{\beta},\Delta}(\omega)} = \frac{Z_\beta}{Z_\alpha} \cdot\frac{\exp\left(\sum_{(\alpha,w) \in \Phi_\alpha} w \cdot Q_\omega(\alpha_i) \right)}{\exp\left(\sum_{(\alpha',w') \in \Phi_\beta} w' \cdot Q_\omega(\alpha') \right)} \\
    &= \frac{Z_\beta}{Z_\alpha} \exp\left( w \cdot C_\Delta \right) \\
    &= \frac{\sum_{\omega' \in \Omega(\sigma,\Delta)}\exp\left(\sum_{(\alpha',w') \in \Phi_\beta} w' \cdot Q_{\omega'}(\alpha') \right)}{\sum_{\omega' \in \Omega(\sigma,\Delta)}\exp\left(\sum_{(\alpha,w) \in \Phi_\alpha} w \cdot Q_{\omega'}(\alpha)\right)} \cdot e^{ w \cdot C_\Delta} \\
    &= \frac{1}{e^{ w \cdot C_\Delta}} \cdot e^{ w \cdot C_\Delta} = 1
\end{align*}
The reasoning for the second case, $Q_{\omega}(\alpha) = -Q_{\omega}(\beta) + C_\Delta$, is completely analogical.
\end{proof}

\begin{proposition}\label{prop:negative_weight}
Let $\Phi = \{(\alpha_1,w_1), \dots, (\alpha_k,w_k) \}$ and $\Phi' = \{(\negation{\alpha_1},-w_1), \dots, (\alpha_k,w_k) \}$. Then, for every domain $\Delta$, and every $\omega \in \Omega(\sigma,\Delta)$, we have:
$$p_{\Phi,\Delta}(\omega) = p_{\Phi',\Delta}(\omega).$$
\end{proposition}
\begin{proof}
The proof follows straightforwardly from Lemma \ref{lemma:constant_difference} above and Proposition \ref{prop:negation} for finite weights and from the definition of semantics of QMLNs for infinite weights.
\end{proof}

\noindent It follows from Proposition \ref{prop:negation} and Proposition \ref{prop:negative_weight} that we can focus on QMLNs that have only positive weights. It also follows that it makes no sense to have a sentence and its negation in the set of sentences defining an QMLN.

\subsection{Limit $w \rightarrow \infty$}

In the seminal  paper on Markov logic networks \cite{RichardsonD06} it was shown that if the weights of formulas of an MLN tend to infinity at the same pace, in the limit the MLN will define a uniform distribution over models of the classical first-order logic theory consisting of the MLN's rules. More precisely, let us denote with $\widehat{\Phi}$ the first-order logic sentence obtained from a given $\Phi = \{(\varphi_1,w), \dots, (\varphi_n,w) \}$ by taking the conjunction of all formulas of the shape $\forall x_1,\dots,x_k\colon \varphi_i(x_1,\dots,x_k)$ where $x_1,\dots,x_k$ are  precisely the free variables in $\varphi_i$. Then the possible worlds that have non-zero probability for $w \rightarrow \infty$ are precisely the models of $\widehat{\Phi}$. 
The next proposition generalizes this by establishing that an analogical property also holds for QMLNs.

\begin{proposition}\label{prop:limit}
Let $\Delta$ be a finite domain and $\Phi(w) = \{(\alpha_1,w), \dots, (\alpha_n,w) \}$ be a QMLN where every weight is $w$, and let $\widehat\Phi$ denote the FOL sentence obtained from 
$\alpha_1\wedge\ldots\wedge\alpha_n$ by replacing every occurrence of $\foralls$ by $\forall$. If $\widehat\Phi$ has a model in $\Omega(\sigma, \Delta)$, then, for every $\omega\in\Omega(\sigma,\Delta)$, we have
\begin{align*}
    \lim_{w \rightarrow \infty} p_{\Phi(w),\Delta}(\omega) = 
    \begin{cases}
    0 & \mbox{if } \omega \not\models \widehat{\Phi}, \\
    \frac{1}{|\{\omega \in \Omega(\sigma,\Delta) \mid \omega \models \widehat{\Phi} \}|} & \mbox{if } \omega \models \widehat{\Phi}.
    \end{cases}
\end{align*}
%
\end{proposition}
\begin{proof}
Let  $\tau(\omega) = \sum_{(\alpha_i,w) \in \Phi} w \cdot Q_\omega(\alpha_i)$. We have, for all worlds $\omega', \omega'' \in \Omega(\sigma, \Delta)$: if $Q_{\omega'}(\alpha) = 1$ for all $(\alpha,w) \in \Phi$ and $Q_{\omega''}(\beta) < 1$ for some $(\beta,w) \in \Phi$, then there is a positive real number $\epsilon$ such that 
$$\frac{\tau(\omega'){-}\tau(\omega'')}{w} = \sum_{(\alpha_i,w) \in \Phi} Q_{\omega'}(\alpha_i) -  \sum_{(\alpha_i,w) \in \Phi} Q_{\omega''}(\alpha_i) \geq \epsilon$$
It is easy to notice that $\tau(\omega')-\tau(\omega'') \rightarrow \infty$ for $w \rightarrow \infty$. Hence also 
$$\frac{p_{\Phi(w),\Delta}(\omega')}{p_{\Phi(w),\Delta}(\omega'')} \rightarrow \infty \mbox{, for } w \rightarrow \infty$$ Since, by Remark \ref{rem:1} above, for all $\alpha$, $Q_\omega(\alpha) = 1 \mbox{ iff } \omega \models \alpha'$ where $\alpha'$ is obtained by replacing all $\foralls$ quantifiers in $\alpha$ by $\forall$, the proposition holds.
\end{proof}

\subsection{Relational Marginal Problems and QMLNs}\label{sec:random-semantics}

From the discussion in Section \ref{sec:RMPs} and Point~\eqref{eq:avg_as_exp} in Definition~\ref{def:sentence_statistics}, it follows that standard MLNs containing only quantifier-free FOL formulas  are solutions of maximum entropy relational marginal problems constrained by sentence statistics of sentences that contain \emph{only} $\foralls$ quantifiers. 
Hence, a natural question is whether the same also holds for relational marginal problems constrained by the more general sentence statistics given by Definition \ref{def:sentence_statistics} which do contain statistical quantifiers $\foralls$ as well as classical quantifiers 
$\forall$ and $\exists$. In the rest of the section we  sketch the argument showing that the answer to this question is positive.

\smallskip We use $(\alpha, q)$ to denote the constraint $Q[\alpha] = q$, where $\alpha$ is a quantified sentence and $q$ is a probability. Let $\Delta$ be a finite domain and $\mathcal{C} = \{(\alpha_1,q_1),\dots,(\alpha_k,q_k) \}$ be a set of  constraints that the sought distribution must satisfy. We require some auxiliary notation. We define the sets 
\begin{align*}
    \mathcal{C}_H &= \{\alpha \mid  (\alpha,1) \in \mathcal{C} \} \cup \{ \negation{\alpha} \mid (\alpha,0) \in \mathcal{C} \} , \\
    \mathcal{C}_S &= \{ (\alpha,w) \mid  (\alpha,w) \in \mathcal{C} \mbox{ and } 0 < w < 1\}.
\end{align*}
We further define a set $\Omega$ of worlds by taking
$$\Omega = \{\omega \in \Omega(\sigma,\Delta) \mid \bigwedge_{\alpha \in \mathcal{C}_H} Q_\omega(\alpha) = 1\}.$$
The motivation behind this is that, in any probability distribution $p(\omega)$ satisfying all constraints $\mathcal{C}$, every possible world \emph{not in $\Omega$} must have probability $0$. We have to treat this separately from the rest of the constraints.\footnote{If we just plugged the constraints from $\mathcal{C}_H$ into the optimization problem, Slater's condition \cite{boyd2004convex} would not hold and we could not establish the duality based on it. \smallskip}

In order to solve the relational marginal problem, we introduce a variable $P_\omega$ for every world $\omega \in \Omega$, intuitively representing the probability of $\omega$.
%
%
%
%
%
The optimization problem representing the maximum entropy relational marginal problem is then given by the objective function
\begin{equation*}
\max_{\{ P_\omega \mid \omega \in \Omega \}}  \sum_{\omega \in \Omega} P_{\omega} \log{\frac{1}{P_\omega}}    
\end{equation*}
subject to the following constraints:
\begin{align*}
&\forall (\alpha_i,q_i) \in \mathcal{C}_S : 
\sum_{\omega \in \Omega} Q_\omega(\alpha_i) \cdot P_\omega = q_i \\
&\forall \omega \in \Omega : P_{\omega} \geq 0, \quad \sum_{\omega \in \Omega} P_{\omega} = 1
\end{align*}

\noindent Assuming\footnote{If this condition is not satisfied we have to add additional hard formulas that explicitly rule out the worlds that have zero probability in every solution satisfying the given marginal constraints.} that there exists a feasible solution of the optimization problem such that $P_\omega > 0$ for all $\omega \in \Omega$, and using standard techniques from convex optimization \cite{boyd2004convex,optimization_entropy_counting}, specifically the construction of Lagrangian dual problems and the use of Slater's condition, we arrive at the solution
\begin{equation}\label{eq:qrmp}
    P_\omega = \frac{1}{Z} \exp{\left(\sum_{(\alpha_i,w_i) \in \mathcal{C}_S} w_i \cdot Q_\omega(\alpha_i) \right)}
\end{equation}
where $Z$ is a normalization constant and the weights $w_i$ are obtained as 
solutions of the optimization problem (the dual of the maximum entropy relational marginal problem) which is to maximize the following expression:
$$
\sum_{(\alpha_i,q_i) \in \mathcal{C}_S} w_i q_i - \log\left({\sum_{\omega \in \Omega} \exp\left({\sum_{(\alpha_i,q_i) \in \mathcal{C}_S} \! \! w_i \cdot  Q_\omega(\alpha_i)}\right)}\right)\label{eq:dual}
$$
As the main result of this section, we have verified that the maximum entropy relational entropy problem for QMLNs has a similar significance as in classical MLNs. In particular, both Points~(i) and~(ii) from Section~\ref{sec:RMPs} are satisfied as well for QMLNs, that is, QMLNs are also solutions of max entropy relational problems constrained by sentence statistics and the result of the latter agrees with maximum likelihood estimation when the input probabilities stem from a domain of the same size.



\section{A Translation for MAP-Inference}\label{sec:tranMAP}

In this section we describe a translation from \emph{arbitrary} quantified MLNs to quantified MLNs that contain the statistical quantifiers $\foralls$ only as a leading prefix. We have already seen in Proposition~\ref{prop:general} that the latter QMLNs correspond essentially to standard MLNs. Since the translation can be performed in polynomial time, the given translation establishes a polynomial time reduction of MAP in QMLNs to MAP in MLNs.

\paragraph{Overview.} The given quantified MLN is translated by processing the weighted sentences one by one. More specifically, we show how to eliminate a single classical quantifier that appears before a block of $\foralls$'s in the quantifier prefix of the quantified sentence. By exhaustively applying this elimination, we end up with a set of weighted sentences where all $\foralls$ quantifiers appear in a prefix block of $\foralls$'s.

\medskip For the description of the elimination, let us suppose that $(\alpha,w)$ is a weighted sentence with $\alpha$ defined as follows
$$Q_1 x_1 \dots Q_k x_k \foralls x_{k+1} \dots \foralls x_{k+l} : \psi(x_1, \dots, x_{k+l})$$
where $Q_i \in \{\exists, \forall, \foralls \}$ for $1 \leq i \leq k-1$, $Q_k \in \{\exists, \forall \}$ and $\psi(x_1,\dots,x_{k+l})$ is a formula with free variables $x_1, \dots, x_{k+l}$; recall that the formula $\psi(x_1,\dots,x_{k+l})$ may also contain variables bound by quantifiers $\exists$ and $\forall$ but not by $\foralls$. The quantified sentence $(\alpha,w)$ is transformed into a set of hard constraints, that is, weighted sentences of the shape $(\varphi,\infty)$ with $\varphi$ a FOL sentence, and a single weighted sentence $(\alpha',w')$ with $w'=|\Delta|\cdot w$ and $\alpha'$ being
\begin{align*}
    Q_1 x_1 \dots Q_{k-1} x_{k-1} \foralls x_k \dots \foralls x_{k+l} : \psi'(x_1, \dots, x_{k+l}) 
\end{align*}
for some formula $\psi'$ to be defined below.
Observe that the effect of the step is to turn quantifier $Q_k$ into $\foralls$.


\paragraph{Eliminating $Q_k x_k$} In order to simplify notation in the description of the elimination step, we will abbreviate $(x_1,\ldots,x_{k-1})$ with $\mathbf{x}$ and $(x_{k+1},\ldots,x_{k+l})$ with $\mathbf{z}$, and write, e.g., $\psi(\mathbf a,a,\mathbf b)$ instead of $\psi[\mathbf x/\mathbf a,x_k/a,\mathbf z/\mathbf b]$.
We describe how to replace $Q_k x_k$ by $\foralls x_k$. By Proposition~\ref{prop:negative_weight}, we can assume without loss of generality that $Q_k$ is in fact $\exists$. 
By the semantics, the variable $x_k$ \emph{maximizes} the sentence statistics for the variables $\mathbf z$ over all possible choices of $a\in\Delta$. Our main idea is to simulate the computation of the sentence statistic in the MLN itself. For this purpose, we introduce a fresh\footnote{In general, we need to introduce fresh predicates names for every transformed $(\alpha_i,w_i)$ separately, e.g. $\textit{max}_{\alpha_i}$ etc. For brevity, we do not show this explicitly in the text.} $k$-ary predicate name $\textit{max}$, set 
$$\psi'(\mathbf x,x_k,\mathbf z) = \textit{max}(\mathbf x,x_k)\wedge\psi(\mathbf{x},x_k,\mathbf{z}),$$
and appropriately define $\textit{max}$ using hard constraints. More formally, let us denote with $\mathsf{Wit}_{\psi,\omega}(\mathbf a,a)$ the set of all assignments of $\mathbf z$ to values $\mathbf b$ such that $\psi(\mathbf a,a,\mathbf b)$ is satisfied in world $\omega$, that is,
$$\mathsf{Wit}_{\psi,\omega}(\mathbf a,a) = \{\mathbf{b} \in \Delta^l \mid \omega \models \psi(\mathbf{a},a,\mathbf{b}) \}.$$
Our goal is to enforce that, in every world $\omega$, $\textit{max}$ satisfies the following property~$(\ast)$:
\begin{itemize}[leftmargin=8mm]
    \item[$(\ast)$] for every choice $\mathbf{a}=(a_1,\ldots, a_{k-1})$ of values for $\mathbf x$, there is precisely one $a^*$ such that $\omega\models\textit{max}(\mathbf{a},a^*)$, and moreover, this 
    $a^*$ satisfies
    \begin{equation}\label{eq:inequality}
      |\mathsf{Wit}_{\psi,\omega}(\mathbf a,a)|\leq |\mathsf{Wit}_{\psi,\omega}(\mathbf a,a^*)|
    \end{equation}
for all $a\in\Delta$.
\end{itemize}
Indeed, property~$(\ast)$ formalizes the mentioned semantics for the sentence statistic for $\exists$. For enforcing it, observe that the inequality~\eqref{eq:inequality} is satisfied iff there is an injective mapping from the set on the left-hand side, $\mathsf{Wit}_{\psi,\omega}(\mathbf{a},a)$, to the set on the right-hand side, $\mathsf{Wit}_{\psi,\omega}(\mathbf{a},a^*)$.
We exploit this observation as follows. First define a collection of linear orders on domain elements, one linear order $\preceq_{\mathbf{a}}$ for each assignment of a tuple $\mathbf{a}$ of domain elements to the variables in $\mathbf{x}$. We represent the order $\preceq_\mathbf{a}$ by the predicates $\textit{leq}(\mathbf{a},\cdot,\cdot)$. The linear orders are enforced by hard constraints. More precisely, we ensure that one such linear order exists for any assignment of domain elements to variables in $\mathbf{x}$ by adding hard constraints for axiomatizing antisymmetry, transitivity, and totality, respectively:
\begin{align}
    \forall \mathbf{x} \forall y,z &\colon \textit{leq}(\mathbf{x}, y, z) \wedge \textit{leq}(\mathbf{x}, z, y) \Rightarrow y = z,\label{eq:leq1}\\
    \forall \mathbf{x} \forall x,y,z &\colon \textit{leq}(\mathbf{x}, x, y) \wedge \textit{leq}(\mathbf{x}, y, z) \Rightarrow \textit{leq}(\mathbf{x}, x, z),\label{eq:leq2}\\
    \forall \mathbf{x} \forall x, y &\colon \textit{leq}(\mathbf{x}, x, y) \vee \textit{leq}(\mathbf{x}, y, x).\label{eq:leq3}
\end{align}

\noindent Next, we connect the linear order construction with the idea of injective mappings described above. This is done via another fresh predicate name $\textit{fn}$ which encodes the required mapping. Intuitively, in $\textit{fn}(\mathbf a, a, a',\mathbf b,\mathbf b')$, $\mathbf{a}$ refers to the current assignment to $\mathbf{x}$, constants $a,a'$ refer to the elements we are interested in for $x_k$, and the function maps $\mathbf{b}\in \mathsf{Wit}_{\psi,\omega}(\mathbf{a},a)$ to $\mathbf{b'}\in \mathsf{Wit}_{\psi,\omega}(\mathbf{a},a')$. We add the following hard constraints:
\begin{align*}
\forall \mathbf{x} \forall y,y'\forall \mathbf{z} &\colon \textit{leq}(\mathbf{x}, y, y') \wedge \psi(\mathbf{x}, y,\mathbf{z}) \\
&\hspace{-.2cm}\Rightarrow \left( \exists \mathbf{z}' \colon \psi(\mathbf{x},y',\mathbf{z}') \wedge \textit{fn}(\mathbf{x},y,y',\mathbf{z},\mathbf{z}') \right), \\[1mm]
\forall \mathbf{x} \forall y,y',z,z',z'' &\colon \textit{fn}(\mathbf{x},y,y',\mathbf{z},\mathbf{z}') \wedge \textit{fn}(\mathbf{x},y,y',\mathbf{z},\mathbf{z}'') \\
&\hspace{3.5cm}
\Rightarrow \mathbf z' = \mathbf z'', \\[1mm]
 \forall \mathbf{x} \forall y,y',z,z',z'' &\colon \textit{fn}(\mathbf{x},y,y',\mathbf{z},\mathbf{z}') \wedge \textit{fn}(\mathbf{x},y,y',\mathbf{z}'',\mathbf{z}') \\
&\hspace{3.5cm}
\Rightarrow \mathbf z = \mathbf z''.
\end{align*}

\noindent The first two sentences enforce that, if $a \preceq_{\mathbf{a}} a'$, then there exists a mapping from $\mathsf{Wit}_{\psi,\omega}(\mathbf{a},a)$ to $\mathsf{Wit}_{\psi,\omega}(\mathbf{a},a')$. Injectivity of the mapping is ensured by the third sentence.

In order to define the predicate $\textit{max}$, we add the following hard constraints:
\begin{align*}
    \forall \mathbf{x} \exists y &\colon \textit{max}(\mathbf{x},y), \\
    \forall \mathbf{x}, y, y' &\colon \textit{max}(\mathbf{x},y) \wedge \textit{leq}(\mathbf{x},y,y') \Rightarrow y = y'.
\end{align*}

\paragraph{Correctness} We have given some intuition above, but let us provide some more details. First, it is not hard to see that the added hard constraints ensure that $\textit{max}$ indeed satisfies the desired property~$(\ast)$.

Now, let $\Delta$ be an arbitrary domain, $\omega^*$ be the most probable world of the QMLN $\Phi$ over domain $\Delta$, and let $\Phi'$ be obtained from $\Phi$ by the application of a single quantifier elimination step. Further, denote with $\sigma'\supseteq\sigma$ the extended vocabulary. We call a world $\omega'\in\Omega(\sigma',\Delta)$ an \emph{extension of $\omega\in\Omega(\sigma,\Delta)$} if for every $R\in \sigma$ of arity $k$ and all $a_1,\ldots,a_k\in\Delta$, we have $\omega' \models R(a_1,\ldots,a_k)$ iff $\omega\models R(a_1,\ldots,a_k)$. 
It is easy to see that our construction ensures that, in fact, every world in $\Omega(\sigma',\Delta)$ is the extension of a (unique!) world in $\Omega(\sigma,\Delta)$, and conversely, every world in $\Omega(\sigma,\Delta)$ has an extension in $\Omega(\sigma',\Delta)$. Moreover, the sentence statistics for $\alpha$ and its replacement $\alpha'$ relate as follows:
\begin{lemma}\label{lem:correct-map}
  Let $\omega'$ be an extension of $\omega$. Then $$Q_{\omega'}(\alpha')=\frac{1}{|\Delta|} Q_\omega(\alpha).$$
\end{lemma}

\begin{proof}
To see this, let $\mathbf{a}$ be some assignment to $\mathbf{x}$ and let $a^*\in\Delta$ be the element that exists by Property~$(\ast)$. We have $Q_{\omega'}(\foralls \mathbf z\colon\psi'(\mathbf a, a^*,\mathbf z))=Q_{\omega}(\foralls \mathbf z\colon\psi(\mathbf a, a^*,\mathbf z))$ and, for all $a\neq a^*$, $Q_{\omega'}(\foralls \mathbf z\colon\psi'(\mathbf a, a,\mathbf z))=0$ since $\omega'\not\models\textit{max}(\mathbf a,a)$ for such $a$. By the semantics, we obtain $$Q_{\omega'}(\foralls x_k\mathbf z\colon\psi'(\mathbf a, x_k,\mathbf z))= \frac{1}{|\Delta|}Q_{\omega}(\exists x_k\foralls \mathbf z\colon\psi(\mathbf a, x_k,\mathbf z))$$ The statement from the Lemma follows from the fact that the constant factor $1/{|\Delta|}$ distributes over $\min,\max$ and expectation in the definition of the semantics of $\forall,\exists$, and $\foralls$.
\end{proof}
The definition of the updated weight $w'=|\Delta|\cdot w$ now implies that there is a constant $c$ such that $p_{\Phi',\Delta}(\omega')=c\cdot p_{\Phi,\Delta}(\omega)$ for all $\omega\in\Omega(\sigma,\Delta)$ and all extensions 
$\omega'\in\Omega(\sigma',\Delta)$ thereof. This establishes the correctness of the  reduction. 
\begin{theorem}
  If $\omega'$ is an extension of $\omega$, then $\omega'$ is a most probable world in $p_{\Phi',\Delta}$ iff $\omega$ is a most probable world in $p_{\Phi,\Delta}$.
\end{theorem}

\section{A Translation for Marginal Inference}\label{sec:tranMARG}

The translation given in the previous section does not quite work for marginal inference. 
Note that for MAP inference, it is enough if all extensions of any possible world of the original problem have the same weight. It is not a problem if there are multiple extensions of the same world or if different worlds have different numbers of extensions (as long as they have the same weight). A single world sometimes has multiple extensions because the linear order $\textit{leq}$ and the functions represented by $\textit{fn}$ are not uniquely defined. However, for marginal inference, this is no longer acceptable. We have to ensure that any two worlds will have the same number of extensions. 


\smallskip We fix these problems by further restricting the functions encoded by $\textit{fn}$ and the order encoded by $\textit{leq}$. More specifically, our goal is to add another set of hard constraints such that
\begin{itemize}[leftmargin=8mm]

    \item[$(\ast\ast)$] every world $\omega\in \Omega(\sigma,\Delta)$ has the same number of extensions $\omega'\in\Omega(\sigma',\Delta)$. 
\end{itemize}
To realize that, we exploit again the idea of the linear order. More specifically, we 
add a fresh binary predicate name $\leq$ and enforce that it is a linear order on $\Delta$ by using hard constraints such as those in Equations~\eqref{eq:leq1}--\eqref{eq:leq3}. Based on $\leq$, we break all possible ties that might occur in the definition of $\textit{max},\textit{leq},\textit{fn}$, in the sense that for a fixed choice of $\leq$, there is exactly one choice of $\textit{max},\textit{leq},\textit{fn}$. First, we enforce that $\textit{fn}$ has the right domain: 
\begin{align*}
\forall\mathbf x\forall x,y\forall \mathbf z,\mathbf z' &\colon \textit{fn}(\mathbf x,x,y,\mathbf z,\mathbf z') 
\\ &\hspace{-.2cm} 
\Rightarrow (\psi(\mathbf{x},x,\mathbf z) \wedge \psi(\mathbf{x},y,\mathbf z'))
\end{align*}
For breaking ties in $\textit{leq}$, we add the following constraint stating that, if $\textit{leq}(\mathbf{x},x,y)$ holds and the function encoded by $\textit{fn}$ is also surjective at given points $\mathbf x,x,y$, then $x\leq y$: 
\begin{align*}
    \forall \mathbf x\forall x,y \colon& \Big(\textit{leq}(\mathbf x,x,y) \wedge{} \\
    & \big(\forall \mathbf z'. \psi(\mathbf x,y,\mathbf z') \Rightarrow \exists \mathbf z. \textit{fn}(\mathbf x,x,y,\mathbf z,\mathbf z')\big)\Big) \\
    & \Rightarrow x\leq y
\end{align*}
Next we enforce that $\textit{fn}$ preserves the order $\leq$ by including the constraint
\begin{align*}
    &\forall \mathbf x\forall x,y \forall \mathbf z_1,\mathbf z_1',\mathbf z_2,\mathbf z_2'\colon \\
    &\quad \Big( \mathbf z_1\leq^* \mathbf z_2\wedge \textit{fn}(\mathbf x,x,y,\mathbf z_1,\mathbf z_1')\wedge \textit{fn}(\mathbf x,x,y,\mathbf z_2,\mathbf z_2')\Big)\\ &\hspace{5.5cm} \Rightarrow \mathbf z_1'\leq^* \mathbf z_2'
\end{align*}
where the order $\leq^*$ is defined -- using straightforward constraints -- as the (unique) lexicographic extension of $\leq$ to the arity of $\mathbf{z}$.
For instance,  consider the sets $\{1,2\}$ and $\{1,2,3\}$. Assuming $1 \leq 2 \leq 3$, this constraint excludes, among others, the function that maps $1$ to $3$ and $2$ to $1$, because $\leq$ is not preserved: $1\leq 2$, but $3\not\leq 1$.

\smallskip Finally, note that  it can still be the case that two worlds have different number of extensions because each of the functions represented by $\textit{fn}$ is order-preserving w.r.t. $\leq^*$ and has a uniquely defined domain but, apart form that, does not have to satisfy any other constraints. For instance the number of such functions from $\{1,2,3 \}$ to $\{1,2,3,4 \}$ and the number of such functions from $\{1,2,3 \}$ to $\{ 1,2,3,4,5 \}$ are different. We address this by requiring that the functions represented by $\textit{fn}$ map every element to the smallest element possible:
\begin{align*}
    & \forall \mathbf{x}\forall x,y \forall \mathbf{z}_1, \mathbf{z}_1' \colon \textit{fn}(\mathbf{x},x,y,\mathbf{z}_1,\mathbf{z}_1') \\ & \Rightarrow \left( \forall \mathbf{z}_2'. \psi(\mathbf{x},y,\mathbf{z}_2') \wedge \mathbf{z}_2' \leq^* \mathbf{z}_1' \Rightarrow \exists \mathbf{z}_2 . \textit{fn}(\mathbf{x},x,y,\mathbf{z}_2,\mathbf{z}_2')\right)
\end{align*}

\paragraph{Correctness} Let $\Phi''$ be the result of adding the described constraints to $\Phi'$. Based on the given intuitions, one can easily show that Property~$(\ast\ast)$ is satisfied. Since there are $|\Delta|!$ possible choices for $\leq$, we get that $$p_{\Phi'',\Delta}(\omega')=\frac{1}{|\Delta|!}p_{\Phi,\Delta}(\omega),$$
for any extension $\omega'$ of $\omega$.
Since $\omega'\models \varphi$ iff $\omega\models\varphi$ for any given sentence $\varphi$ over $\sigma$, we obtain the desired result:
\begin{theorem}\label{thm:correct-marg} For every FOL sentence $\varphi$, we have $$P_{\omega \sim p_{\Phi,\Delta}}(\varphi) = P_{\omega \sim p_{\Phi'',\Delta}}(\varphi).$$
\end{theorem}

\medskip
\noindent An important question is whether the result from Theorem~\ref{thm:correct-marg} can be extended
to computing a marginal query $Q_{\Phi,\Delta}[\alpha]$ for a quantified sentence $\alpha$. The answer to this question is positive. We next outline how this is done. 
Let $\Phi$ be the MLN resulting from applying the translation above  to a given QMLN. We distinguish two cases. 

\smallskip \noindent $(i)$ If $\alpha$ contains only $\foralls$ quantifiers  as a leading prefix of the block of quantifiers, we can borrow techniques from \cite{vanhaaren.mlj}. In particular, we need to create the partition of the groundings $\alpha\vartheta$ of $\alpha$; specifically, groundings of the  variables bound by  $\foralls$, such that the probability of any two groundings $\alpha\vartheta$ and $\alpha\vartheta'$ in the same equivalence class of this partition is equal, i.e. $P_{\omega \sim p_{\Phi,\Delta}}(\alpha\vartheta) = P_{\omega \sim p_{\Phi'',\Delta}}(\alpha\vartheta')$. This partitioning can be achieved using preemptive shattering~\cite{poole2011towards}. Once we have the partitioning, we can compute the marginal probability $P_{\omega \sim p_{\Phi,\Delta}}(\alpha\vartheta)$ for one representative $\alpha\vartheta_i$ from each partition class $\mathcal{P}_i$. Finally, we can compute $Q_{\Phi,\Delta}[\alpha]$ as follows:
$$Q_{\Phi,\Delta}[\alpha] = \frac{1}{|\bigcup_{\mathcal{P}_i \in \mathcal{P}}\mathcal{P}_i|} \sum_{\mathcal{P} \in \mathcal{P}_i} P_{\omega \sim p_{\Phi,\Delta}}(\alpha\vartheta_i) \cdot |\mathcal{P}_i|$$

\noindent $(ii)$ If $\alpha$ contains an arbitrary quantifier prefix, we first need to convert it to the form assumed above, that is,   we apply the transformation described in this section to obtain a sentence $\alpha'$ that contains $\foralls$ quantifiers only as a prefix while also generating several additional hard constraints, which we include in  $\Phi$. We can then use the procedure outlined in $(i)$ above to compute $Q_{\Phi,\Delta}[\alpha]$.






\section{On QMLNs Restricted to Two Variables}\label{sec:TwoVar}

It has been shown that marginal inference for MLNs can often be reduced to \emph{symmetric weighted first-order model counting (WFOMC)}, see e.g.\ \cite{BeameBGS15}. In this context of particular importance is the \emph{two-variable fragment of   FOL (FO$^2$)}, since for FO$^2$ symmetric WFOMC can be solved in polynomial time data complexity, that is, when  the formula is considered fixed and the only input is the domain $\Delta$~\cite{Broeck11,BroeckMD14,BeameBGS15}. Given the fact that
MLNs containing formulae with up two variables ($2$-MLNs) can be encoded as WFOMC for FO$^2$, $2$-MLNs are \emph{domain liftable}.
Hence, a natural question to ask is whether the same holds for  \emph{quantified} MLNs. 

Let us first remark that the reduction described in  Section~\ref{sec:tranMARG} does not preserve the quantifier rank. Indeed, consider the example formula $\exists x\foralls y\colon\mathsf{knows}(x,y)$. It is originally an FO$^2$ formula, but the elimination step introduces a \emph{quaternary} predicate name $\textit{fn}$ and the constraints for this predicate require using \emph{four} variables. As a consequence, we cannot `reuse' the  results on WFOMC for FO$^2$ and thus attain  domain liftability  for $2$-QMLNs. Moreover, our reduction explicitly introduces transitivity axioms while there are only a few known very restricted cases where WFOMC is domain liftable in the presence of transitivity \cite{kazemi.nips.16}.

In the light of   recent results by~\citeauthor{KuLuLICS18}~(\citeyear{KuLuLICS18}) on WFOMC for an extension of FO$^2$ with \emph{counting quantifiers}, it is not very surprising that a straightforward translation \emph{preserving} the quantifier rank from QMLNs to MLNs seems not possible.  To see this, note that   to compute the sentence statistic, we need to take into account the \emph{out-degree} of  domain elements, where the out-degree of an element is the  number of  elements that are related to it, cf.\ Example~\ref{example:first}. 
Indeed, \cite{KuLuLICS18} put quite some technical effort, relying on sophisticated model-theoretic techniques and combinatorics, 
to show that WFOMC for FO$^2$ with \emph{one functional axiom}\footnote{Intuitively, enforcing out-degree 1 for a particular binary relation, that is, elements are related with at most one element.} is polynomial time in data complexity. The complexity of WFOMC for  FO$^2$ with many functional axioms or more generally, with arbitrary counting quantifiers remains a challenging open problem. This gives an insight on the difficulty   of studying the computational complexity of reasoning in QMLNs restricted to two variables, making it a research challenge by itself, which we plan to address in the future. In particular, we will investigate  the exact connection between 2-QMLNs  and WFOMC for extensions of FO$^2$ with some means for counting.

\section{Related Work}\label{sec:relwork}

Classical FOL  quantifiers $(\forall, \exists)$ were already considered in the original work on Markov Logic Networks~\cite{RichardsonD06}, albeit without a rigorous definition. A precise treatment of  FOL quantifiers was carried out later on in~\cite{BroeckMD14,VdBFTDB17}. In particular,  \citeauthor{BroeckMD14}~(\citeyear{BroeckMD14}) show how to remove existential quantifiers while preserving marginal inference results. In all these works, however,  MLNs with quantifiers were defined in a way that is equivalent to QMLNs with a prefix of $\foralls$ quantifiers. As a consequence,  it is not possible to directly represent statistics ({\em features} in MLN parlance) that correspond to sentences in which $\forall$ or $\exists$ precedes $\foralls$ in the quantifier block. 

There has been also some work on other types of aggregation. For example, some works considered explicit constructs for counting in relational models~\cite{MilchZKHK08,JainBB10}.  
In another research effort,  \citeauthor{LowdD07}~(\citeyear{LowdD07}) introduced recursive random fields  that are capable of emulating certain forms of more complex aggregation. However, recursive random fields do not seem capable of even representing statistics such as $\exists x \foralls y : \mathsf{knows}(x,y)$. Finally,  \citeauthor{BeltagyE15} (\citeyear{BeltagyE15}) studied the effect of the domain closure assumption on the semantics of probabilistic logic when existential quantifiers are allowed.

In the context of \emph{probabilistic soft logic (PSL)}, \citeauthor{FarnadiBMGC17} (\citeyear{FarnadiBMGC17}) recently introduced soft quantifiers based on quantifiers from \emph{fuzzy logic}. However, their approach strictly applies to fuzzy logic. In particular,  in PSL random variables e.g. $\mathsf{smokes}(\mathsf{Alice})$, 
may acquire non-Boolean truth values. 

Another stream of research that is related to our work is the study of the effect of domain size and its extrapolation on the probability distributions encoded using various relational learning systems \cite{PooleBKKN14,KazemiBKNP14,kuzelka.aaai.2018}. However, none of these works studied the interplay of statistical and classical quantifiers.

There has been also some related work within KR research, e.g.\ about statistical reasoning in description logics~\cite{penaloza2017towards,LutSchroe-KR10}.



\section{Discussion and Future Work}\label{sec:con}

In this paper, we have investigated  the  extension QMLNs of MLNs with statistical quantifiers, allowing to express e.g.\ measures on the proportion of domain elements fulfilling certain property. We developed some key  foundations by establishing a relation between MLNs and QMLNs. In particular, we provided a polytime reduction of the standard reasoning tasks MAP and MARG in QMLNs  to their counterpart in MLNs. Furthermore, we also showed how to generalize the random substitution semantics  to QMLNs.

As for future work, it might be interesting to develop more direct approaches to MAP and MARG in QMLNs. Indeed, even if the developed translations provide polytime reductions of reasoning in QMLNs to reasoning in MLNs (and overall,
a good understanding of the relation between QMLNs and MLNs), they do not yield an immediate  practical approach since  the introduction of new predicates with greater arity is required. 
Another interesting aspect of future work is to investigate the statistical properties of QMLNs. For MLNs with quantifier-free formulas, \citeauthor{kuzelka.aaai.2018} (\citeyear{kuzelka.aaai.2018}) derived bounds on expected errors of the statistics' estimates. However, obtaining similar bounds for the more general statistics considered here seems considerably more difficult because of the minimization and maximization that are involved in them.

\section*{Acknowledgments}

The authors were  supported by  EU's Horizon 2020 programme under the Marie Sk{\l}odowska-Curie grant 663830,  ERC consolidator
grant 647289 CODA, and the Research Foundation - Flanders (project G.0428.15), respectively.

\bibliographystyle{aaai}
\bibliography{main}
\clearpage

\end{document}